\documentclass{article} 
\usepackage{iclr2027_conference,times}

\usepackage{amsmath,amsfonts,bm}

\def\eqref#1{equation~\ref{#1}}

\def\1{\bm{1}}

\DeclareMathAlphabet{\mathsfit}{\encodingdefault}{\sfdefault}{m}{sl}
\SetMathAlphabet{\mathsfit}{bold}{\encodingdefault}{\sfdefault}{bx}{n}

\usepackage[pagebackref=false,breaklinks=true,letterpaper=true,colorlinks,bookmarks=false]{hyperref}
\usepackage{url}
\usepackage{amsmath,amssymb,amsthm}
\usepackage{booktabs}
\usepackage{graphicx}
\usepackage{multirow}
\usepackage{placeins}
\usepackage[T1]{fontenc}
\usepackage{array}

\usepackage{pgfplots}
\pgfplotsset{compat=1.18}
\usepgfplotslibrary{groupplots,fillbetween}
\definecolor{cNavy}{HTML}{17324D}
\definecolor{cTeal}{HTML}{329A99}
\definecolor{cViolet}{HTML}{8272B5}
\definecolor{cAmber}{HTML}{E3A13B}
\definecolor{cBlue}{HTML}{3F6FA6}
\definecolor{cGray}{HTML}{7A8390}
\pgfplotsset{
  paperaxis/.style={
    tick label style={font=\scriptsize, color=cNavy},
    label style={font=\scriptsize, color=cNavy},
    title style={font=\scriptsize, color=cNavy, at={(0,1)}, anchor=south west, yshift=-1pt},
    legend style={font=\tiny, draw=none, fill=none, legend cell align=left},
    axis line style={cNavy, thin}, axis lines=left,
    tick style={cNavy}, every axis plot/.append style={line width=0.9pt},
  }
}

\newtheorem{proposition}{Proposition}
\newtheorem{lemma}{Lemma}

\newcommand{\x}{\mathbf{x}}
\newcommand{\z}{\mathbf{z}}
\newcommand{\Pinit}{P^{\mathrm{init}}}
\newcommand{\nullst}{\varnothing}
\newcommand{\msg}{\mathcal{E}}
\newcommand{\freeE}{\mathcal{F}}
\newcommand{\gibbs}{\mathcal{P}}

\usepackage[capitalize]{cleveref}
\crefname{section}{Sec.}{Secs.}
\Crefname{section}{Section}{Sections}
\Crefname{table}{Table}{Tables}
\crefname{appendix}{App.}{Apps.}
\Crefname{appendix}{Appendix}{Appendices}

\title{UnfoldCRF: Structured Mask Refinement with Image-Conditioned Latent Regions}

\author{
Chunming He$^{1}$\,,
Rihan Zhang$^{1}$\,,
Lei Xu$^{2}$\,,
Guanyi Qin$^{3}$\,, \\
\textbf{ Chengyu Fang}$^4$\,,
\textbf{Longxiang Tang}$^{5}$\,,
\textbf{Fengyang Xiao}$^{1,*}$\,,
\textbf{Sina Farsiu$^{1,*}$}\\
	$^1$Duke University, ~$^2$ EPFL, ~$^3$ National University of Singapore, \\ ~$^4$Tsinghua University, ~ $^5$Harvard University
\\
~ The first two authors contributed equally,
$*$ Corresponding Author. \\
~ Contact: chunming.he@duke.edu
 }

\iclrfinalcopy 

\begin{document}

\maketitle

\begin{abstract}
\looseness=-1
Learned mask refiners improve segmentation accuracy, but it is hard to tell how much of the improvement comes from explicit structure rather than from extra capacity, and whether it holds up when the mask generator or its error distribution changes. UnfoldCRF treats refinement as inference in a conditional random field over pixel labels and latent region variables. Its energy has a corrected unary term, learned local pairwise interactions, and image-conditioned latent-region consistency, with a null state that lets a region with weak label agreement withdraw from the consistency term; inference unrolls damped mean-field updates on this one energy. To isolate the effect of structure, we compare against recurrent black-box refiners that read the same inputs and receive the same parameter budget, stage count, and supervision. On COD10K, UnfoldCRF beats the strongest matched control by 1.0 $F^\omega_\beta$ point, improves all four COD metrics, and lowers the fraction of images made worse from 11.7\% to 8.5\%. Under a train-once protocol over five datasets and several mask sources, the 2.6M-parameter variant gains 4.2 mean $\Delta$IoU against 2.0 for its control, and a variant built on frozen DINOv2 features matches the strongest foundation-model refiner with about a seventh of its resident parameters while staying ahead of its own control. On mask generators never seen in training, the gain is 2.0 $F^\omega_\beta$ points against 0.6 for the control. Zeroing individual messages shows where the corrections come from: the pairwise messages mostly fix boundaries, the region messages mostly fix non-boundary errors. Code and supporting materials will be publicly released.
\end{abstract}

\setlength{\abovedisplayskip}{2pt}
\setlength{\belowdisplayskip}{2pt}
\section{Introduction}
\label{sec:intro}

Segmentation networks rarely fail as independent pixel noise. On targets entangled with their context, such as camouflaged objects, transparent surfaces, clinical structures, and thin or fragmented shapes, the failures have structure: interiors are hollowed out, components are dropped or hallucinated, and boundaries drift coherently along low-contrast contours \citep{fan2020camouflaged,xie2020trans10k,he2023weakly}. Refinement methods post-process an initial mask to remove such errors, from dense CRFs with prescribed affinity functions \citep{krahenbuhl2011efficient} to learned refiners built on cascades \citep{cheng2020cascadepsp}, boundary correction \citep{yuan2020segfix}, diffusion \citep{wang2023segrefiner,shen2025uncertainty}, and prompted foundation models \citep{linsamrefiner,ke2023hqsam,price2026promptmoe}. Deep unfolding has also been used for iterative refinement, though recent segmentation-oriented unfolding methods derive their updates from reconstruction objectives rather than from an energy over the labels themselves \citep{he2025run,he2025nested}.

Most learned refiners predict mask updates with generic network modules and no explicit energy over the labels. When such a refiner improves on the upstream mask, one cannot tell whether structure did the work or whether capacity, iteration, and supervision did, nor whether the improvement would survive a different mask generator or a different error distribution. We ask: \textbf{under carefully matched inputs, parameter count, iteration count, and supervision, does explicit structured inference provide more reliable refinement than black-box iterative refinement, both in-distribution and under shifts in mask generators and error distributions?} By \emph{reliable} we mean a positive average gain, few images made worse, and a gain that survives shift. A second question is whether the corrections can be traced to particular energy messages.

\textbf{UnfoldCRF} is an image-conditioned latent-region CRF for mask refinement. Its energy has a corrected unary term, a symmetric local pairwise term in a learned feature space, and a mass-normalized consistency term over latent region variables. Learned local soft incidences connect pixels to regions at several granularities; a region variable takes either a semantic label or a null state, and the null state reduces the influence of regions whose members disagree. The construction builds on robust $P^N$ potentials and associative hierarchical CRFs \citep{kohli2009robust,ladicky2009associative,arnab2016higher}, but the pixel--region structure is image-conditioned and learned end to end through unrolled inference. For a given input, every stage runs on the same conditional energy; only a stage-dependent damping coefficient changes.

A refiner can raise in-distribution accuracy in two ways: by learning what valid masks look like, or by learning the error habits of the generators it was trained on. The pairwise and latent-region terms encode the first kind of knowledge, local boundary consistency and regional label agreement, and the unary term keeps the upstream prediction as evidence rather than re-predicting it. When the generator changes, these relations should be less tied to its particular error patterns than a learned update rule is. \Cref{sec:analysis} tests this with held-out generators and message interventions.

Our main contributions are:

\textbf{(1) Structural model.} An image-conditioned latent-region CRF for mask refinement with local learnable soft incidences, symmetric pairwise interactions, and a mass-normalized region-consistency term whose latent state is semantic or null; its MAP reduction is a truncated region-consistency potential with a vote-fraction threshold (Proposition~\ref{prop:map}).

\textbf{(2) Shared-energy unfolding.} A damped mean-field refinement process in which all stages share one conditional energy and only the damping coefficient varies. Damping preserves the fixed points of the shared system (Lemma~\ref{lem:fixedpoint}), and the latent-region message passing is sparse, costing $O(RNsL)$ per stage for $N$ pixels, $R$ levels, $s$ candidate regions per pixel, and $L$ labels.

\textbf{(3) Empirical findings.} Against parameter-matched recurrent controls, UnfoldCRF improves COD10K $F^\omega_\beta$ by 1.0 point and makes fewer images worse; on a five-benchmark train-once suite its S and L variants gain 4.2 and 6.3 mean $\Delta$IoU against 2.0 and 4.6 for their controls. On unseen mask generators it keeps a 2.0-point $F^\omega_\beta$ gain where the control keeps 0.6, and zeroing messages shows that the pairwise messages mainly correct boundary errors while the region messages mainly correct non-boundary false negatives and false positives.

\begin{figure}[t]
\centering
\setlength{\abovecaptionskip}{-0.1cm}
\includegraphics[width=\linewidth]{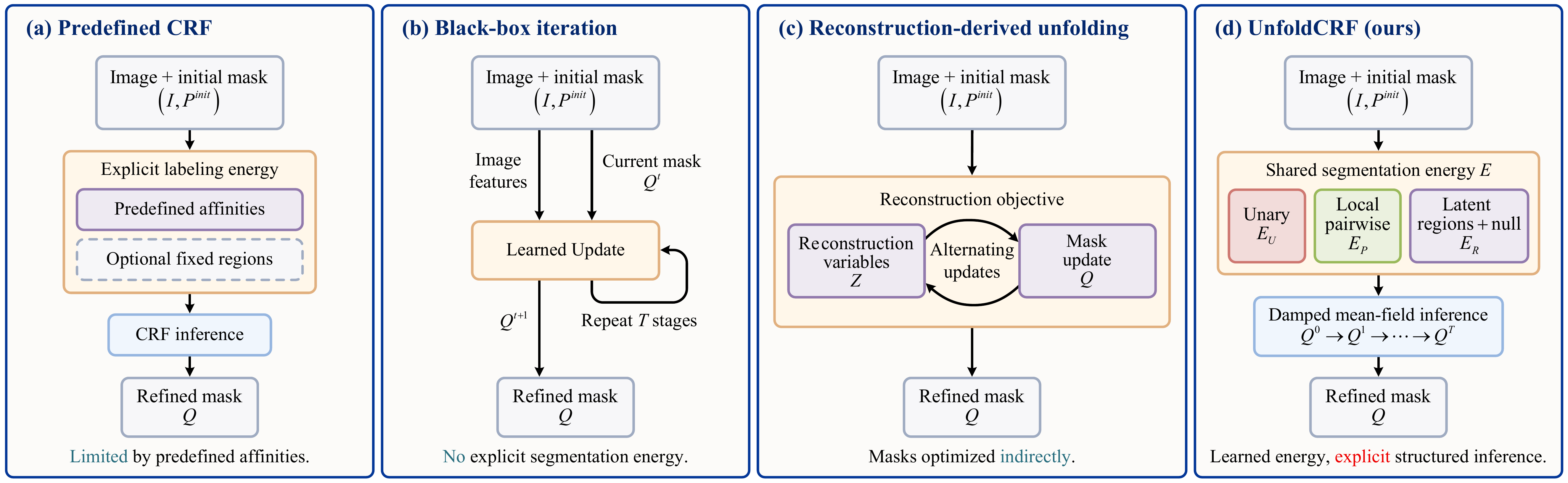}
\caption{\textbf{Refinement formulations.}
(a) Predefined CRFs use prescribed affinities and optional fixed regions.
(b) Black-box refiners learn iterative mask updates without an explicit energy.
(c) Reconstruction-derived unfolding alternates reconstruction and mask updates.
(d) UnfoldCRF performs damped mean-field inference on a shared, image-conditioned segmentation energy with corrected unary, local pairwise, and latent-region terms.}\vspace{-4mm}
\label{fig:StruComp}
\end{figure}

\section{Related work}
\label{sec:related}

\textbf{Higher-order and hierarchical CRFs.}
Robust $P^N$ potentials encode segment-level label consistency with a truncated penalty and admit an exact representation through auxiliary segment variables with a free label \citep{kohli2009robust}; associative hierarchical CRFs stack such auxiliary variables across segmentation layers \citep{ladicky2009associative}. Filter-based mean-field inference extends dense pairwise models to certain higher-order terms \citep{vineet2014filter}. We keep the latent-variable form of these potentials but make the pixel--region associations soft, local, image-conditioned, and trainable through unrolled inference.

\textbf{Neural and trainable CRF inference.}
CRF-as-RNN unrolls parallel mean-field inference for dense Gaussian CRFs into a trainable recurrent module \citep{zheng2015conditional}, while subsequent work recasts CRF message passing using convolutional and pixel-adaptive filtering \citep{teichmann2018convolutional,su2019pixel}. Higher-order detection and superpixel potentials have also been incorporated end to end \citep{arnab2016higher}. Exact inference for Gaussian CRFs can instead be obtained by solving linear systems \citep{chandra2016fast}. Parameter learning and convergent inference for dense CRFs are studied in \citep{krahenbuhl2013parameter}; principled parallel variational inference with convergence guarantees is developed in \citep{baque2016principled}, while Regularized Frank--Wolfe provides a broader optimization framework that includes mean field as a special case \citep{lehuu2021regularized}. Learnable high-dimensional filtering is explored in \citep{jampani2016learning}. Deep unfolding has also been applied to segmentation through reconstruction-derived objectives \citep{he2025run,he2025nested}, whereas the unrolled updates here are those of a labeling energy.

\textbf{Differentiable region grouping.}
SLIC provides a classical hard superpixel construction \citep{achanta2012slic}, while Superpixel Sampling Networks learn task-specific soft superpixels with local candidate assignment \citep{jampani2018superpixel}; recent work further studies differentiable superpixel formation and connectivity \citep{juybari2026differentiable}. Slot Attention \citep{locatello2020object} learns soft pixel-to-slot associations in the same spirit. Our incidences are computed once from each image--mask pair and held fixed; what changes across stages is the label distribution of each region.

\looseness=-1
\textbf{Learned mask refinement.}
CascadePSP combines global and local refinement \citep{cheng2020cascadepsp}, SegFix corrects boundaries \citep{yuan2020segfix}, and SegRefiner formulates refinement as a discrete diffusion process \citep{wang2023segrefiner}. For camouflaged object detection, UMBD uses uncertainty-guided Bernoulli diffusion to selectively refine uncertain regions while exploiting features from the upstream segmenter \citep{shen2025uncertainty}. Foundation-model-based refinement builds on SAM-style promptable segmentation \citep{kirillov2023segment,ke2023hqsam}: SAMRefiner derives noise-tolerant prompts from coarse masks \citep{linsamrefiner}, while PromptMoE uses mixture-of-experts prompting for refinement across tasks and domains \citep{price2026promptmoe}. Phoenix trains on semantically structured mask errors generated by adversarial perturbation and corrects them with contrastive refinement learning \citep{kim2026phoenix}. We differ from these refiners in the question asked: refinement here is inference in a shared segmentation energy, and the effect of that structure is isolated with matched controls.

\section{Method}
\label{sec:method}

\subsection{A conditional random field over masks and latent regions}
\label{sec:crf}

Let $I$ be an image and let $\mathcal{G}$ denote the inference grid with $N$ spatial sites (stride 4 in our implementation); the initial per-site label distribution $\Pinit \in \Delta_L^N$ is the output of an arbitrary upstream model over $L$ classes ($L{=}2$ for binary tasks), resized to $\mathcal{G}$. All site-indexed quantities below live on $\mathcal{G}$. We model refinement as approximate inference in a conditional random field over pixel labels $\x = (x_1,\dots,x_N)$ and latent region labels $\z = \{z^{(r)}_k : 1 \leq r \leq R,\; k \in \mathcal{K}^+_r\}$, one per active region $k$ at each of $R$ granularity levels, where the active-region sets $\mathcal{K}^+_r$ are defined in \cref{sec:region}:
\begin{equation}
E(\x,\z \mid I, \Pinit) \;=\; E_U(\x) + E_P(\x) + E_R(\x,\z),
\label{eq:energy}
\end{equation}
conditioned jointly on $I$ and $\Pinit$. A lightweight encoder consumes $(I,\Pinit)$ and predicts the input-dependent quantities in \cref{eq:energy}: unary corrections, pairwise interaction weights, and pixel--region incidences; the global energy parameters, the label compatibility $\mu$ and the per-level strengths $\beta_r$ and truncations $\kappa_r$, are learned jointly and shared across stages, as is the incidence temperature $\tau_r$ of \cref{eq:assign}. The damping coefficients $\alpha_t$ belong to the inference procedure (\cref{sec:inference}), not to the energy. The input-dependent quantities are computed once and held fixed across inference stages, so every stage performs inference on the same conditional energy; they remain differentiable, so gradients reach the encoder through the unrolled computation, with the exception of the active-region mask, which is fixed and treated as stop-gradient (App. \ref{app:protocols}).

\textbf{Corrected unary.}
The unary term uses the upstream prediction as its base and adds an image-conditioned residual:
\begin{equation}
E_U(\x) = \sum_i \psi^U_i(x_i), \qquad
\psi^U_i(l) = -\log \Pinit_i(l) + \delta_{\theta,i}(l),
\label{eq:unary}
\end{equation}
where $\Pinit$ is clipped to $[\epsilon_0, 1-\epsilon_0]$ and renormalized ($\epsilon_0 = 10^{-4}$) so that hard upstream masks have finite log-evidence, and $\delta_{\theta}$ is a per-pixel, per-label residual predicted by the encoder.

\textbf{Symmetric local pairwise term.}
For neighboring sites we use a signed, image-conditioned pairwise interaction:
\begin{equation}
E_P(\x) = \sum_{i<j} K_{ij}\, \mu(x_i, x_j), \qquad
K_{ij} = \tfrac{1}{2}\big(\widehat{K}_{ij} + \widehat{K}_{ji}\big)\,\mathbf{1}\!\left[j \in \Omega(i)\right],
\label{eq:pairwise}
\end{equation}
where $\Omega(i)$ is the union of $3{\times}3$ windows at dilations $\{1,2,4\}$ around pixel $i$ and $\widehat{K}_{ij} = g_\theta(h_i, j-i)$ is a signed interaction weight predicted by a $1{\times}1$ convolution on the encoder features $h_i$, one output per offset $j-i \in \Omega$.
The window relation is symmetric ($j \in \Omega(i)$ iff $i \in \Omega(j)$) and the explicit symmetrization gives $K_{ij}=K_{ji}$; the learnable label compatibility is symmetric by parameterization, $\mu = \tfrac{1}{2}(\hat\mu + \hat\mu^\top)$, so \cref{eq:pairwise} is a well-defined energy over unordered pixel pairs. The pairwise term stays local; longer-range interactions go through the latent regions. Its mean-field message excludes the self contribution,
\begin{equation}
\msg^{P}_i(l) = \sum_{l'} \mu(l, l') \sum_{j \neq i} K_{ij}\, Q_j(l').
\label{eq:pwmsg}
\end{equation}

\begin{figure}[t]
\centering
\setlength{\abovecaptionskip}{0cm}
\includegraphics[width=\linewidth]{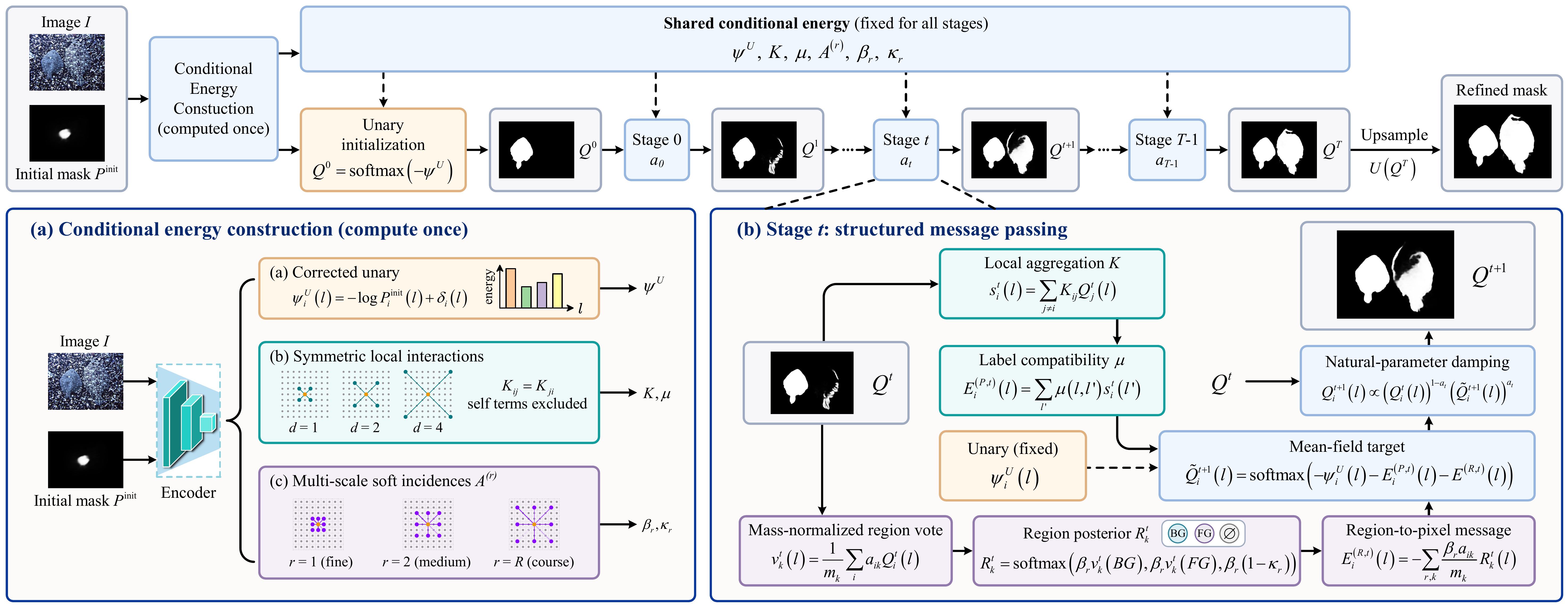}
\caption{\textbf{Framework of UnfoldCRF.}
Top: $I$ and $\Pinit$ define a conditional segmentation energy that is shared by all $T$ stages and initializes $Q^0$; only the damping coefficient $\alpha_t$ varies by stage.
Bottom left: the encoder constructs the corrected unary potential, symmetric local interactions, and multi-scale soft pixel--region incidences.
Bottom right: stage $t$ forms the region posterior $R^t$ from $Q^t$, computes the pairwise and mass-normalized latent-region messages, and combines them with the fixed unary term to obtain $\widetilde{Q}^{t+1}$; natural-parameter damping then produces $Q^{t+1}$.}\vspace{-4mm}
\label{fig:Framework}
\end{figure}

\subsection{Latent-region robust consistency}
\label{sec:region}

\textbf{Soft incidences and active regions.}
All levels share the same inference grid $\mathcal{G}$ but use different region-center densities. At level $r$, the encoder yields pixel embeddings $e^{(r)}_i \in \mathbb{R}^{C}$, $K_r$ region centers are placed on a regular grid, and the center embedding $u^{(r)}_k$ is the average of the pixel embeddings in grid cell $k$.
Following the local-candidate design of differentiable superpixels \citep{jampani2018superpixel}, each pixel is softly assigned only to the $s{=}9$ centers of its $3{\times}3$ neighboring cells $\mathcal{N}_r(i)$:
\begin{equation}
a^{(r)}_{ik} = \frac{\exp\!\big(-\lVert e^{(r)}_i - u^{(r)}_k \rVert^2 / \tau_r\big)}{\sum_{k' \in \mathcal{N}_r(i)} \exp\!\big(-\lVert e^{(r)}_i - u^{(r)}_{k'} \rVert^2 / \tau_r\big)} \;\;\text{for } k \in \mathcal{N}_r(i), \qquad a^{(r)}_{ik} = 0 \;\;\text{otherwise},
\label{eq:assign}
\end{equation}
with a learned per-level temperature $\tau_r > 0$. Rows of $A^{(r)}$ sum to one, the scheme is the same for all tasks and label counts, and across the $R$ levels incidence construction costs $O(RNsC)$. Because $e^{(r)}$ is computed from the encoder features of $(I,\Pinit)$, incidences are conditioned on both the image and the initial mask. Let $m^{(r)}_k = \sum_i a^{(r)}_{ik}$ be the region mass and $\mathcal{K}^+_r = \{k : m^{(r)}_k \geq m_{\min}\}$ the set of active regions; regions outside $\mathcal{K}^+_r$ carry no latent variable and neither enter the energy nor send messages. Since incidences are fixed before inference, $\mathcal{K}^+_r$ is fixed across stages.

\textbf{Robust region energy with a null state.}
Each active region has a latent state $z^{(r)}_k \in \{1,\dots,L\} \cup \{\nullst\}$, where $\nullst$ allows the region to opt out of consistency. With strength $\beta_r = \mathrm{softplus}(b_r) > 0$ and truncation $\kappa_r = \sigma(\gamma_r) \in (0,1)$, the mass-normalized region energy is
\begin{equation}\hspace{-3mm}
E_R(\x,\z) = \sum_{r=1}^{R} \sum_{k \in \mathcal{K}^+_r} \frac{\beta_r}{m^{(r)}_k} \sum_i a^{(r)}_{ik}\, c_{\kappa_r}\big(x_i, z^{(r)}_k\big),
\qquad
c_{\kappa}(x_i, z_k) =
\begin{cases}
\mathbf{1}[x_i \neq z_k], & z_k \neq \nullst,\\[2pt]
\kappa, & z_k = \nullst.
\end{cases}
\label{eq:regionE}
\end{equation}
Mass normalization expresses region disagreement as a weighted label fraction; after eliminating the latent state by MAP, the region penalty is capped at $\beta_r \kappa_r$ independently of soft mass (\cref{prop:map}).

\textbf{Mean-field updates.}
We use the factorization $q(\x,\z) = \prod_i Q_i(x_i) \prod_{r} \prod_{k \in \mathcal{K}^+_r} R^{(r)}_k(z^{(r)}_k)$ and suppress the level index $(r)$ when writing the updates of a single level. With the mass-normalized vote $v_k(l) = \frac{1}{m_k} \sum_i a_{ik}\, Q_i(l)$, the exact mean-field updates of \cref{eq:regionE} are (derivation in App. \ref{app:deriv})
\begin{equation}
R_k(l) \,\propto\, \exp\!\big(\beta_r\, v_k(l)\big),
\qquad
R_k(\nullst) \,\propto\, \exp\!\big(\beta_r (1-\kappa_r)\big),
\label{eq:Rupdate}
\end{equation}
\begin{equation}
\msg^{R}_i(l) \;=\; -\sum_{r=1}^{R} \sum_{k \in \mathcal{K}^+_r} \frac{\beta_r\, a^{(r)}_{ik}}{m^{(r)}_k}\, R^{(r)}_k(l) \;+\; \mathrm{const},
\label{eq:Rmsg}
\end{equation}
where the constant is independent of $l$ and cancels in the pixel softmax. The region branch therefore consists of mass-normalized sparse pixel--region aggregation, a softmax over region states, and mass-normalized sparse backprojection to pixels, at $O(RNsL)$ cost per stage.

\begin{proposition}[MAP reduction to normalized robust truncation]
\label{prop:map}
Fix a hard labeling $\x$ and an active region $k$ at level $r$, suppressing the level superscript for readability, and let $p_k(l) = \frac{1}{m_k}\sum_i a_{ik}\, \mathbf{1}[x_i = l]$ denote
the mass-weighted fraction of region $k$ carrying label $l$. Then
\begin{equation}
\min_{z_k} E_{R,k}(\x, z_k) \;=\; \beta_r\, \min\big\{\kappa_r,\; 1 - \max\nolimits_l p_k(l)\big\}.
\label{eq:map}
\end{equation}
\end{proposition}
A proof is given in App. \ref{app:deriv}. \cref{eq:map} is a robust $P^N$-style potential over soft regions \citep{kohli2009robust}: the penalty grows with label disagreement and is truncated at $\beta_r \kappa_r$ through the null state. At finite $\beta_r$ the posterior $R_k$ is soft, so $1-\kappa_r$ is the vote fraction at which a semantic state becomes preferred over the null state ($v_k(l) > 1-\kappa_r$), not a deterministic gate.

\subsection{Unrolled damped mean-field inference}
\label{sec:inference}

Inference initializes $Q^{0}_i = \mathrm{softmax}(-\psi^U_i)$, which differs from $\Pinit_i$ through the learned correction, and unrolls $T$ stages indexed $t = 0,\dots,T-1$. Stage $t$ executes, in order,
\begin{align}
R^{t} &\leftarrow \text{\cref{eq:Rupdate} evaluated at } Q^{t} \text{ for all active regions},\\
\widetilde{Q}^{t+1}_i(l) &\propto \exp\!\big[-\psi^U_i(l) - \msg^{P,t}_i(l) - \msg^{R,t}_i(l)\big],\\
\log Q^{t+1}_i(l) &= (1-\alpha_t)\, \log Q^{t}_i(l) + \alpha_t \log \widetilde{Q}^{t+1}_i(l) - \log Z^t_i,
\label{eq:damp}
\end{align}
where $\msg^{P,t}$ and $\msg^{R,t}$ are the messages of \cref{eq:pwmsg,eq:Rmsg} computed from $Q^t$ and $R^t$, and $\alpha_t = \sigma(\eta_t) \in (0,1)$ is a per-stage damping coefficient with learnable $\eta_t$. All energy parameters, embeddings, and incidences are shared across stages; only $\alpha_t$ is stage dependent. $Q^T$ is the final grid-level marginal, $\mathcal{U}(Q^T)$ is the final refined mask, and $R^T$ is formed from $Q^T$ for diagnostics.

\begin{lemma}[Damping preserves fixed points]
\label{lem:fixedpoint}
For every $\alpha \in (0,1]$, the damped update in \cref{eq:damp} and the undamped mean-field update of the conditional energy in \cref{eq:energy} have the same set of joint fixed points $(Q^\ast, R^\ast)$, provided that $Q^\ast$ lies in the interior of the probability simplex.
\end{lemma}
App. \ref{app:deriv} proves \cref{lem:fixedpoint}; the iterates remain in the interior because $\psi^U$ and all messages are finite. Parallel mean field need not monotonically decrease the variational free energy \citep{baque2016principled,krahenbuhl2013parameter}. We therefore report $\freeE(Q^t,R^t)$ and the undamped fixed-point residual $\mathrm{Res}_{\mathrm{FP}}(t) = \frac{1}{N}\sum_i \lVert \widetilde{Q}^{t+1}_i - Q^{t}_i \rVert_1$ as empirical inference diagnostics (\cref{fig:dynamics}).

\subsection{Training and configuration}
\label{sec:training}

Let $\mathcal{U}$ denote bilinear upsampling from $\mathcal{G}$ to the image resolution followed by renormalization, and define $\ell(Q, Y) = \mathcal{L}_{\mathrm{CE}}(\mathcal{U}(Q), Y) + \mathcal{L}_{\mathrm{Dice}}(\mathcal{U}(Q), Y)$ for ground truth $Y$. The refiner is trained with
\begin{equation}
\mathcal{L} \;=\; \ell(Q^T, Y) \;+\; \frac{1}{2 (T-1)} \sum_{t=1}^{T-1} \ell(Q^t, Y),
\label{eq:loss}
\end{equation}
where $Q^0$ is not supervised and $\Pinit$ is resized to $\mathcal{G}$ by bilinear interpolation (App. \ref{app:protocols}).
Because the refiner reads only $(I, \Pinit)$, it applies to any upstream model whose output mask is available.

\section{Experiments}
\label{sec:experiments}

\begin{table}[t]
\setlength{\abovecaptionskip}{1mm}
\centering
\resizebox{\linewidth}{!}{
\setlength{\tabcolsep}{0.6mm}
\begin{tabular}{lccccccc}
\toprule
Method & Resident params & BIG & VOC & DAVIS585 & ECSSD & MSRA-B & Mean \\
\midrule
Unrefined (IoU / BIoU) & n/a & 78.3/70.1 & 66.7/60.1 & 80.1/83.0 & 81.4/70.2 & 75.2/61.9 & 76.3/69.1 \\
\midrule
CascadePSP~\citep{cheng2020cascadepsp} & 68M & +5.0/+6.3 & +1.7/+0.7 & $-$1.3/$-$1.5 & +0.6/+1.0 & +0.9/+2.7 & +1.4/+1.9 \\
SegRefiner~\citep{wang2023segrefiner} & 119M & \textbf{+9.6}/\textbf{+12.5} & $-$3.9/$-$3.1 & $-$10.9/$-$9.1 & $-$15.0/$-$21.4 & $-$10.7/$-$16.2 & $-$6.2/$-$7.4 \\
DualSight~\citep{price2025dualsight} & 641M & +3.9/+4.6 & +3.3/+6.3 & +1.5/$-$0.6 & +2.0/+4.5 & +2.1/+6.8 & +2.6/+4.3 \\
SAMRefiner~\citep{linsamrefiner} & 641M & +6.8/+9.5 & +7.1/+9.7 & +3.3/+2.0 & +5.1/+9.7 & +4.7/+10.4 & +5.4/+8.3 \\
PromptMoE~\citep{price2026promptmoe} & $\geq$641M & +8.5/+11.0 & +7.9/+10.4 & \textbf{+3.6}/\textbf{+2.4} & \textbf{+6.0}/+10.7 & +5.1/\textbf{+10.5} & +6.2/+9.0 \\
\midrule
Attn.\ recurrent-S (matched) & 2.6M & +2.4/+2.9 & +2.7/+3.1 & +0.6/+0.2 & +2.3/+3.6 & +2.1/+3.9 & +2.0/+2.7 \\
UnfoldCRF-S (ours) & 2.6M & +5.1/+6.9 & +5.2/+7.1 & +2.4/+1.6 & +4.3/+7.4 & +4.0/+8.1 & +4.2/+6.2 \\
Attn.\ recurrent-L (matched) & 92M & +6.2/+8.0 & +5.9/+7.4 & +2.1/+1.2 & +4.6/+7.6 & +4.2/+7.8 & +4.6/+6.4 \\
UnfoldCRF-L (ours) & 92M & +8.7/+11.4 & \textbf{+8.1}/\textbf{+10.6} & +3.4/+2.3 & \textbf{+6.0}/\textbf{+10.8} & \textbf{+5.3}/+10.4 & \textbf{+6.3}/\textbf{+9.1} \\
\bottomrule
\end{tabular}}
\caption{\textbf{General open refinement} on the five-benchmark suite of \citet{price2026promptmoe}: $\Delta$IoU / $\Delta$BIoU (\%) over the unrefined masks (absolute scores in the first row). Published baselines use the reported high-resolution variants under the same coarse-mask protocol. Bold: best per column. The last four rows are our runs, reported as means over 5 training seeds.}
\label{tab:general}
\vspace{-3mm}
\end{table}

\begin{table}[t]
\setlength{\abovecaptionskip}{1mm}
\centering
\resizebox{\linewidth}{!}{
\setlength{\tabcolsep}{1.9mm}
\begin{tabular}{lccccc}
\toprule
& & CHAMELEON & CAMO & COD10K & NC4K  \\ \cmidrule(lr){3-6}  
\multicolumn{1}{l}{\multirow{-2}{*}{Method}} & \multicolumn{1}{c}{\multirow{-2}{*}{Source}} & $M$ / $F^\omega_\beta$ / $E_\phi$ / $S_\alpha$ & $M$ / $F^\omega_\beta$ / $E_\phi$ / $S_\alpha$ & $M$ / $F^\omega_\beta$ / $E_\phi$ / $S_\alpha$ & $M$ / $F^\omega_\beta$ / $E_\phi$ / $S_\alpha$ \\
\midrule
SINet V2 & TPAMI'21 & .029/.792/.922/.890 & .071/.733/.875/.822 & .036/.668/.867/.820 & .048/.769/.898/.848 \\
$+$SegRefiner  & NeurIPS'23 & .073/.619/.812/.777 & .116/.583/.763/.736 & .079/.487/.764/.714 & .075/.635/.830/.781 \\
$+$Phoenix & ECCV'26 & .030/.851/.915/.866 & .079/.768/.834/.781 & .034/\textbf{.766}/.872/.818 & .049/\textbf{.822}/.883/.836 \\
$+$UMBD & arXiv'25 & .027/\textbf{.853}/\textbf{.955}/.892 & .067/.774/.885/.829 & .032/.733/.909/.825 & .044/.807/.911/.849 \\
$+$UnfoldCRF-S & Ours & \textbf{.026}/.849/.953/\textbf{.898} & \textbf{.066}/\textbf{.775}/\textbf{.890}/\textbf{.834} & \textbf{.031}/.740/\textbf{.910}/\textbf{.832} & \textbf{.043}/.812/\textbf{.915}/\textbf{.855} \\
\midrule
FEDER & CVPR'23 & .030/.824/.941/.888 & .073/.740/.866/.799 & .032/.713/.899/.822 & .045/.796/.909/.847 \\
$+$SegRefiner & NeurIPS'23 & .041/.751/.908/.850 & .081/.683/.848/.785 & .043/.635/.866/.794 & .050/.742/.896/.832 \\
$+$Phoenix & ECCV'26 & .034/.846/.925/.861 & .081/.760/.818/.773 & .032/\textbf{.779}/.884/.824 & .047/\textbf{.826}/.886/.836 \\
$+$UMBD & arXiv'25 & .028/.838/.949/.890 & .069/.757/.871/.807 & .030/.732/.905/.827 & .043/.809/.912/.849 \\
$+$UnfoldCRF-S & Ours & \textbf{.026}/\textbf{.848}/\textbf{.953}/\textbf{.896} & \textbf{.067}/\textbf{.763}/\textbf{.879}/\textbf{.813} & \textbf{.029}/.745/\textbf{.910}/\textbf{.834} & \textbf{.042}/.815/\textbf{.917}/\textbf{.855} \\
\bottomrule
\end{tabular}}
\caption{\textbf{Concealed target refinement} on COD under the UMBD~\citep{shen2025uncertainty} protocol: $M\downarrow$ / $F^\omega_\beta\uparrow$ / $E_\phi\uparrow$ / $S_\alpha\uparrow$ on four test sets with SINetV2~\citep{fan2021concealed} and FEDER~\citep{he2023feder} upstreams. SegRefiner~\citep{wang2023segrefiner} and UMBD results are quoted from UMBD; Phoenix~\citep{kim2026phoenix} is evaluated zero-shot on the same coarse masks, and our predictions use the same metric implementation. UMBD additionally uses upstream features. }
\label{tab:concealed}
\vspace{-3mm}
\end{table}

\newlength{\qualw}
\newlength{\qualh}
\setlength{\qualw}{0.115\textwidth}
\setlength{\qualh}{0.095\textwidth}      
\newcommand{\qualimg}[1]{\includegraphics[width=\qualw,height=\qualh]{"#1"}}
\newcommand{\rowlabel}[1]{\rotatebox{90}{\scriptsize\makebox[\qualh][c]{#1}}}

\begin{figure*}[t]
\centering
\setlength{\abovecaptionskip}{1mm}
\vspace{-5mm}
\includegraphics[width=\textwidth]{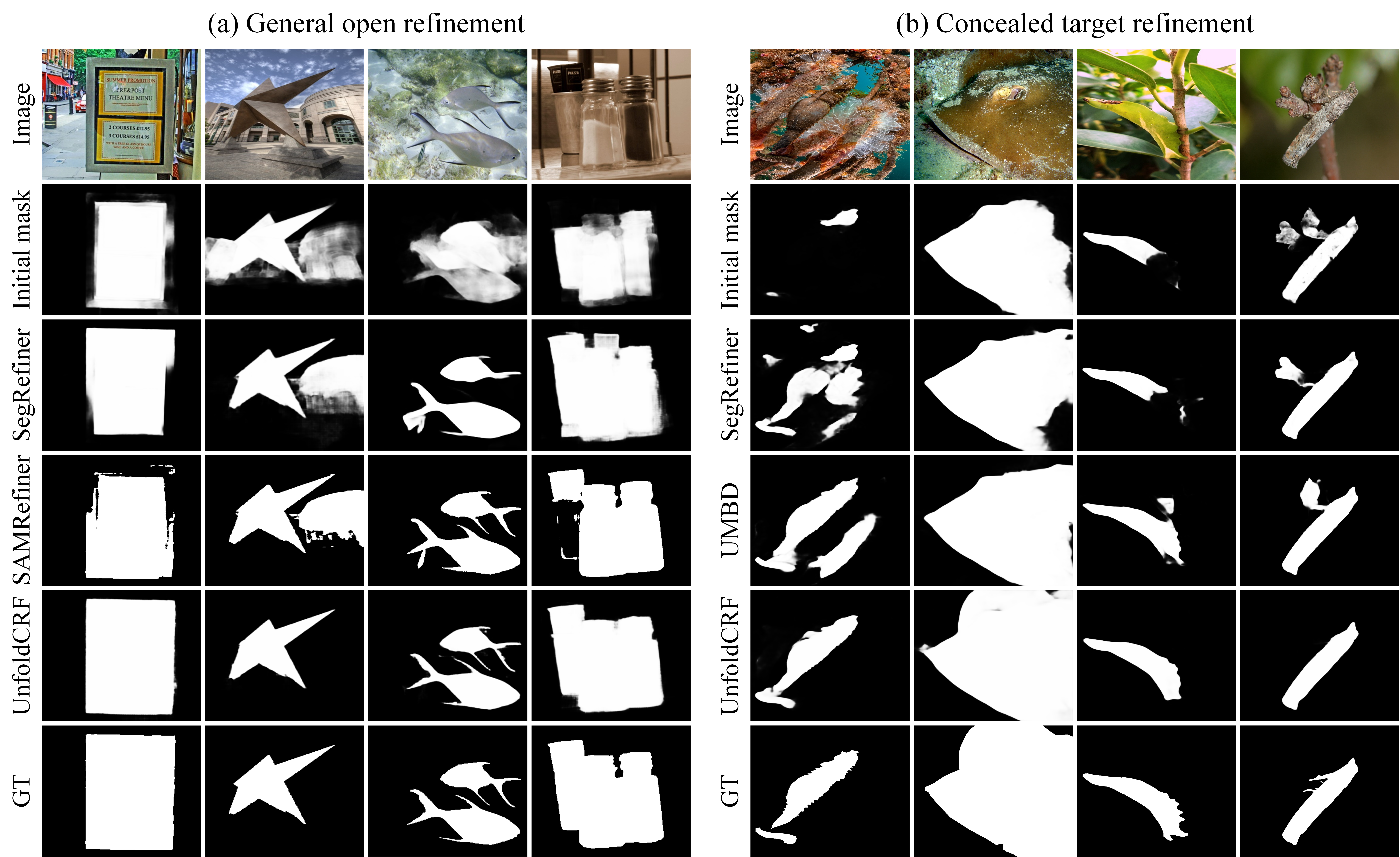}
\vspace{-4.5mm}
\caption{\textbf{Qualitative refinement.}
(a) General open refinement on representative examples.
(b) Concealed target refinement on representative COD examples.
Columns show four examples; rows show the image, initial mask, SegRefiner, SAMRefiner/UMBD, UnfoldCRF, and ground truth.}
\label{fig:qual}
\vspace{-8mm}
\end{figure*}

\subsection{Experimental setup}
\label{sec:exp_setup}

\textbf{Settings.}
\emph{General open refinement} uses the five-benchmark train-once suite of \citet{price2026promptmoe} (BIG \citep{cheng2020cascadepsp}, VOC \citep{everingham2010pascal}, DAVIS585 \citep{chen2022focalclick}, ECSSD~\citep{yan2013hierarchical}, MSRA-B~\citep{liu2010learning}) with its original coarse-mask sources: DeepLabV3 \citep{chen2017rethinking}, FCN-R50 \citep{long2015fully}, and LR-ASPP-MobileNetV3 \citep{howard2019searching} for BIG/VOC; RMBG-1.4/2.0 and BiRefNet-matting for ECSSD and MSRA-B; and the SP and STM masks of DAVIS585. UnfoldCRF and its matched control are trained once on VOC 2012 train \citep{everingham2010pascal} and evaluated frozen; we report $\Delta$IoU and $\Delta$BIoU \citep{cheng2021boundary} over the initial masks. \emph{Concealed refinement} follows UMBD \citep{shen2025uncertainty}: training uses CAMO \citep{le2019anabranch} and COD10K \citep{fan2020camouflaged}, and evaluation uses CHAMELEON \citep{skurowski2018animal}, CAMO, COD10K, and NC4K \citep{lv2021simultaneously}, with SINetV2 \citep{fan2021concealed} and FEDER \citep{he2023feder} as upstream generators and $M$, $F^\omega_\beta$ \citep{margolin2014evaluate}, $E_\phi$ \citep{fan2018enhanced}, and $S_\alpha$ \citep{fan2017structure} as metrics. Unless otherwise stated, controlled ablations and mechanism analyses use COD10K with FEDER as upstream. Further protocol and training details are given in App. \ref{app:protocols}.

\textbf{Model variants and controls.}
UnfoldCRF-S is the standalone model of \cref{sec:method} (custom encoder, 2.6M parameters, $352^2$ input) and is used for all controlled analyses. UnfoldCRF-L keeps the energy and inference operator unchanged but uses a frozen DINOv2-B/14 encoder \citep{oquab2023dinov2} whose intermediate features are fused with a local high-resolution branch at $448^2$ input ($\sim$92M resident and $\sim$6M trainable parameters; App. \ref{app:protocols}). For each variant, the matched attention control uses the same encoder and fusion architectures, with weights trained independently, the same stage count $T$, losses, and supervision, and replaces only the structured operator by $Q^{t+1}=F_{\mathrm{attn}}(h,Q^t)$. For UnfoldCRF-S, the total parameter budget is matched; for UnfoldCRF-L, the frozen DINOv2 backbone is shared and the trainable non-backbone budget is matched.

\subsection{Comparison with state-of-the-art refiners}
\label{sec:sota}

\textbf{General open refinement.}
\Cref{tab:general} reports the five-benchmark suite. Among the published refiners, the SAM-based methods gain most, CascadePSP-Slow gains little, and SegRefiner-HR transfers poorly outside BIG. UnfoldCRF-S gains 4.2 mean $\Delta$IoU and 6.2 mean $\Delta$BIoU with 2.6M resident parameters, 2.2 points in $\Delta$IoU above its matched control (paired 95\% CI $[+1.6,+2.8]$; App.~\ref{app:protocols}). With frozen DINOv2-B features, UnfoldCRF-L reaches 6.3 mean $\Delta$IoU and 9.1 mean $\Delta$BIoU, comparable to PromptMoE on both averages with about one-seventh of the resident parameters, and remains 1.7 $\Delta$IoU points above its matched control (paired 95\% CI $[+1.1,+2.3]$). The 1.7-point gap remains when both models use the same frozen DINOv2 backbone, which rules out backbone strength as the sole explanation. Both variants gain most on BIG and VOC and least on DAVIS585, whose SP and STM initial masks are already accurate.

\textbf{Concealed target refinement.}
\Cref{tab:concealed} reports all four COD metrics because they rank the refiners differently. Phoenix has the highest weighted $F$-measure on COD10K and NC4K but lowers $E_\phi$ or $S_\alpha$ below the upstream mask in several settings, and SegRefiner lowers all four metrics everywhere. UnfoldCRF-S improves all four metrics on all eight upstream--dataset pairs, with the best MAE on all eight, the best $E_\phi$ on seven, and the best $S_\alpha$ on all eight. Unlike UMBD, it uses only the image and the initial mask rather than upstream hidden features; \cref{fig:qual} shows representative corrections of boundary drift, missing structures, and false positives in both settings.

\subsection{Ablation study}
\label{sec:ablation}

\Cref{tab:ablation} evaluates the contributions of the unary correction, the learned pairwise term, the latent regions, and the null state, with parameter counts reported for each variant. Structure controls replace learned interactions or regions with RGB--XY affinities \citep{krahenbuhl2011efficient} or SLIC regions \citep{achanta2012slic,arnab2016higher}. Matched recurrent controls replace structured inference with generic recurrence under the same parameter budget and stage count; the reconstruction-unfolding control alternates reconstruction-state and mask updates in the style of RUN \citep{he2025run} under the same budget (App.~\ref{app:protocols}). The full model improves $F^\omega_\beta$ by 1.0 point over the matched attention control (95\% CI $[+0.6, +1.4]$, App.~\ref{app:protocols}) while reducing Harm from 11.7\% to 8.5\%; replacing learned structure with fixed structure, or removing the null state, lowers every metric.

\begin{table}[t]
\setlength{\abovecaptionskip}{1mm}
\centering
\resizebox{\linewidth}{!}{
\setlength{\tabcolsep}{3.1mm}
\begin{tabular}{lcccccc}
\toprule
Operator & Params & $M\downarrow$ & $F^\omega_\beta\uparrow$ & $E_\phi\uparrow$ & $S_\alpha\uparrow$ & Harm$\downarrow$ \\
\midrule
\multicolumn{7}{l}{\emph{Components}} \\
Upstream (FEDER) & n/a & .032 & .713 & .899 & .822 & n/a \\
Unary only & 2.4M & .031 & .725 $\pm$ .003 & .903 & .826 & 15.6 $\pm$ 0.6 \\
Unary $+$ learned pairwise & 2.5M & .030 & .736 $\pm$ .002 & .906 & .829 & 11.2 $\pm$ 0.5 \\
Unary $+$ learned latent regions & 2.5M & .030 & .738 $\pm$ .003 & .907 & .830 & 10.7 $\pm$ 0.5 \\
Full UnfoldCRF-S ($T{=}5$) & 2.6M & \textbf{.029} & \textbf{.745} $\pm$ .002 & \textbf{.910} & \textbf{.834} & \textbf{8.5 $\pm$ 0.4} \\
Without null state & 2.6M & .031 & .737 $\pm$ .003 & .906 & .829 & 10.1 $\pm$ 0.5 \\
\midrule
\multicolumn{7}{l}{\emph{Structure controls}} \\
Unary $+$ fixed pairwise & 2.4M & .031 & .731 $\pm$ .003 & .904 & .827 & 13.0 $\pm$ 0.6 \\
Unary $+$ fixed SLIC regions & 2.4M & .031 & .732 $\pm$ .004 & .905 & .828 & 13.5 $\pm$ 0.7 \\
\midrule
\multicolumn{7}{l}{\emph{Black-box and formulation controls}} \\
Conv recurrent (matched) & 2.6M & .031 & .733 $\pm$ .004 & .905 & .828 & 12.3 $\pm$ 0.6 \\
Attention recurrent (matched) & 2.6M & .030 & .735 $\pm$ .003 & .906 & .829 & 11.7 $\pm$ 0.5 \\
Reconstruction-unfolding control (RUN-style) & 2.6M & .032 & .712 $\pm$ .005 & .885 & .808 & 10.9 $\pm$ 0.5 \\
\bottomrule
\end{tabular}}
\caption{\textbf{Controlled comparison} on COD10K with FEDER upstream (5 seeds; SD for $F^\omega_\beta$ and Harm). Component ablations share the encoder and training setting; structure controls use predefined pairwise or region structure; recurrent controls match the stage count and parameter budget. Harm is the percentage of images whose IoU decreases after refinement.}
\label{tab:ablation}
\vspace{-3mm}
\end{table}
\begin{figure}[t]
\centering
\includegraphics[width=\linewidth]{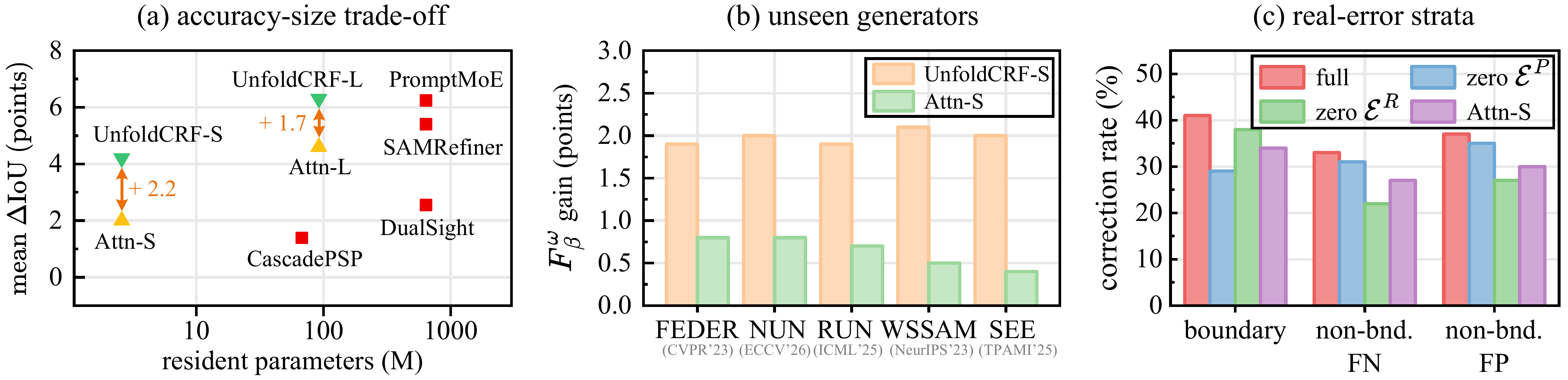}
\vspace{-7.5mm}
\caption{\textbf{Efficiency, generalization, and mechanism.}
(a) Mean $\Delta$IoU versus resident parameters; arrows indicate gains over matched recurrent controls at both scales.
(b) $F^\omega_\beta$ gains on held-out generators.
(c) Correction rates by error type under message interventions.}
\label{fig:generalization}
\vspace{-3mm}
\end{figure}
\subsection{Further analysis}
\label{sec:analysis}

\textbf{Generalization and mechanism.}
\Cref{fig:generalization}a plots the accuracy--parameter trade-off of \cref{tab:general}; panels (b,c) turn to generalization across generators and to mechanism. \Cref{fig:generalization}b evaluates unseen-generator generalization on COD with FEDER, NUN \citep{he2025nested}, RUN \citep{he2025run}, WSSAM \citep{he2023weakly}, and SEE \citep{he2025segment}. For each held-out generator, the refiner is trained on the remaining generators and evaluated frozen, comparing the $F^\omega_\beta$ gain of UnfoldCRF-S and Attn-S (paired difference 1.3 points, 95\% CI $[+0.9,+1.9]$). \Cref{fig:generalization}c partitions initially wrong pixels into boundary errors, non-boundary false negatives, and non-boundary false positives (definitions in App.~\ref{app:protocols}), and reports the fraction corrected after refinement. Zeroing the pairwise messages reduces boundary correction by 12 points but changes non-boundary FN/FP correction by only 2 points, whereas zeroing the region messages reduces non-boundary FN/FP correction by 11/10 points but boundary correction by only 3 points; the two branches act on different error types.

\begin{figure}[t]
\centering
\includegraphics[width=\linewidth]{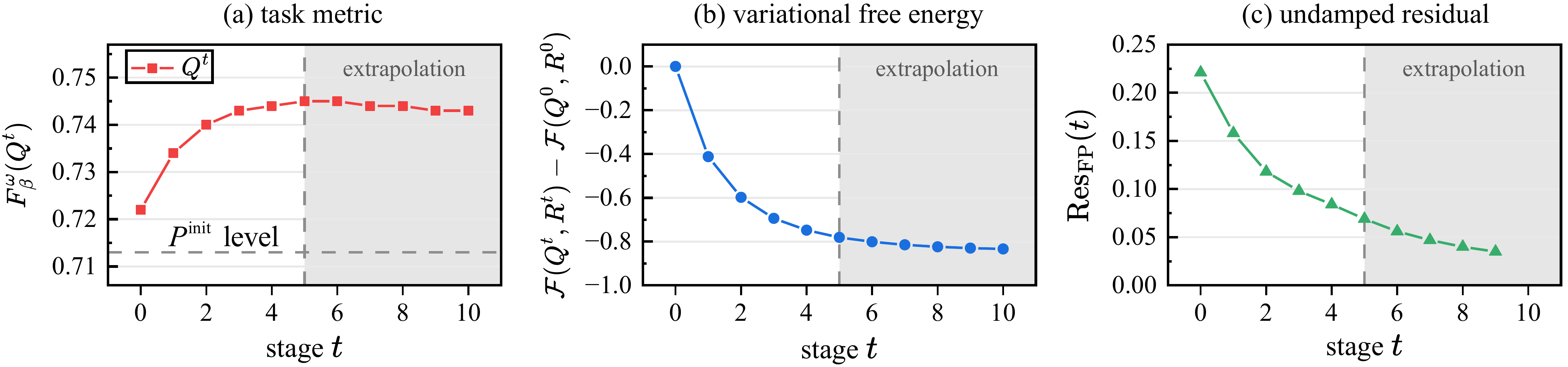}
\vspace{-7.5mm}
\caption{\textbf{Stage-wise inference dynamics} of one default checkpoint ($T{=}5$) on COD10K: (a) $F^\omega_\beta$ of intermediate marginals with the initial-mask level as reference, (b) variational free energy relative to $Q^0$ (test-set mean of the per-image difference), and (c) undamped fixed-point residual. Vertical dashed line: trained depth; shading: test-time extrapolation beyond $T$ with $\alpha_t:=\alpha_{T-1}$. Residuals are recorded before each update for $t = 0, \ldots, 2T-1$.}
\label{fig:dynamics}
\vspace{-3mm}
\end{figure}

\newlength{\failw}
\newlength{\failh}
\begin{figure}[h]
\centering
\setlength{\abovecaptionskip}{1mm}
\vspace{-1.5mm}
\includegraphics[width=\linewidth]{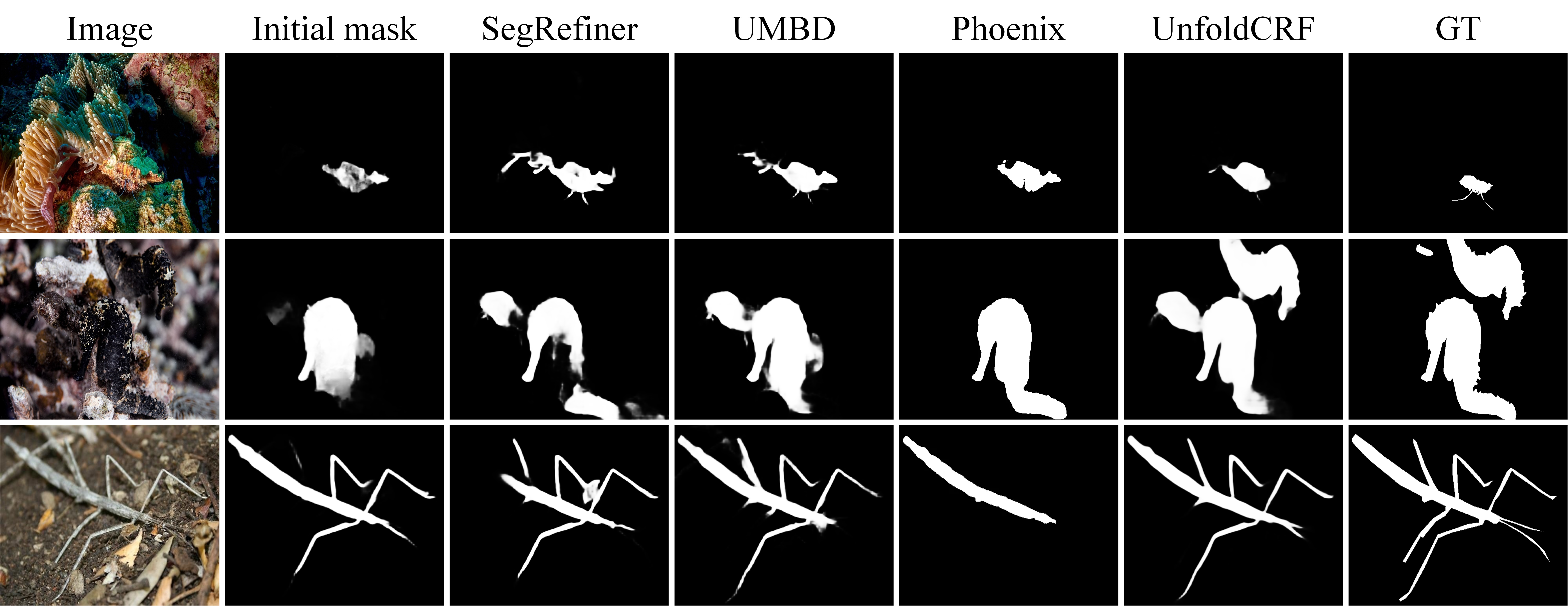}
\vspace{-4.5mm}
\caption{\textbf{Representative failure cases} with FEDER upstream. Columns show the image, initial mask, SegRefiner, UMBD, Phoenix, UnfoldCRF-S, and ground truth. Rows show a confidently mislabeled region, a missing object part, and a thin structure that remains difficult to recover.}
\label{fig:failure}\vspace{-4.5mm}
\end{figure}

\textbf{Stage-wise inference dynamics.}
\Cref{fig:dynamics} follows one default checkpoint through $Q^0,\ldots,Q^T$ and, beyond the trained depth, through test-time extrapolation to $Q^{2T}$ with $\alpha_t:=\alpha_{T-1}$. The gap between $\Pinit$ and $Q^0=\mathrm{softmax}(-\psi^U)$ is the learned unary correction, before any structured iteration. The task metric peaks near the trained depth and drifts slightly downward under further extrapolation, while the free energy and the undamped fixed-point residual keep decreasing beyond it.

\vspace{-1mm}
\section{Discussion and limitations}\vspace{-1mm}
\label{sec:discussion}

\textbf{Structured inference beyond representation strength.}
The controlled comparisons separate two things that a black-box refiner entangles: which pixels and regions interact, which the encoder decides, and how labels change under those interactions, which the shared energy decides. A 1.7-point gap survives when both models sit on the same frozen DINOv2 backbone, so the backbone cannot be the whole story. Zeroing the pairwise or region messages affects different error types. For the checkpoint examined in \cref{fig:dynamics}, segmentation accuracy peaks near the trained depth despite continued decreases in variational free energy.

\textbf{Limitations.}
UnfoldCRF still depends on the quality of the upstream mask and on the resolution of its stride-4 inference grid: confidently wrong regions, missing object parts, and very thin structures are often not recovered (\cref{fig:failure}). UnfoldCRF-L also inherits the resident parameter cost of its DINOv2 backbone. Except for the train-once setting in \cref{sec:sota}, we train a separate refiner for each task, and our experiments are limited to binary masks although the formulation supports $L$ labels; multi-class, instance-level, and temporal settings remain future work.

\vspace{-1mm}
\section{Conclusion} \vspace{-1mm}

UnfoldCRF refines masks by unrolling damped mean-field inference in one image-conditioned segmentation energy. Matched recurrent controls, held-out generators, and message interventions separate its gain from parameter count and backbone strength, and they show that the pairwise term mostly corrects boundary errors while the region term mostly corrects non-boundary ones.

\newpage

\subsection*{AI use statement}
Generative AI tools were used to assist with editing and revising the manuscript. The authors reviewed all AI-assisted changes and take full responsibility for the final content.

\subsection*{Ethics statement}

This work uses publicly available benchmark datasets under their respective terms of use and involves no human subjects, personal data collection, or annotation by the authors. As a mask refiner, UnfoldCRF can inherit errors from the upstream segmenter and may also worsen individual predictions; we therefore report image-harm rates alongside aggregate accuracy. The method should be validated for the target domain before use in safety-relevant applications.

\subsection*{Reproducibility statement}

The conditional energy, mean-field updates, and damping scheme are specified in \cref{sec:method}, with derivations and proofs in App. \ref{app:deriv}. We also provide input provenance, evaluation conventions, implementation and training details, and the construction of the matched controls in App. \ref{app:protocols}.

\bibliography{main}

@String(CVPR= {IEEE Conf. Comput. Vis. Pattern Recog.})

@String(ICCV= {Int. Conf. Comput. Vis.})

@String(ECCV= {Eur. Conf. Comput. Vis.})

@String(BMVC= {Brit. Mach. Vis. Conf.})

@String(ICLR = {Int. Conf. Learn. Represent.})

@String(CVPR  = {CVPR})

@String(ICCV  = {ICCV})

@String(ECCV  = {ECCV})

@String(BMVC  =	{BMVC})

@String(ICLR  = {ICLR})

@inproceedings{cheng2020cascadepsp,
  title={Cascadepsp: Toward class-agnostic and very high-resolution segmentation via global and local refinement},
  author={Cheng, Ho Kei and Chung, Jihoon and Tai, Yu-Wing and Tang, Chi-Keung},
  booktitle={CVPR},
  pages={8890--8899},
  year={2020}
}

@inproceedings{yuan2020segfix,
  title={Segfix: Model-agnostic boundary refinement for segmentation},
  author={Yuan, Yuhui and Xie, Jingyi and Chen, Xilin and Wang, Jingdong},
  booktitle={ECCV},
  pages={489--506},
  year={2020}
}

@inproceedings{arnab2016higher,
  title={Higher order conditional random fields in deep neural networks},
  author={Arnab, Anurag and Jayasumana, Sadeep and Zheng, Shuai and Torr, Philip HS},
  booktitle={ECCV},
  pages={524--540},
  year={2016},
  organization={Springer}
}

@article{teichmann2018convolutional,
  title={Convolutional CRFs for semantic segmentation},
  author={Teichmann, Marvin TT and Cipolla, Roberto},
  journal={BMVC},
  year={2018}
}

@article{he2025nested,
  title={Nested unfolding network for real-world concealed object segmentation},
  author={He, Chunming and Zhang, Rihan and Zhang, Dingming and Xiao, Fengyang and Fan, Deng-Ping and Farsiu, Sina},
  journal={ECCV},
  year={2026}
}

@inproceedings{fan2017structure,
  title={Structure-measure: A new way to evaluate foreground maps},
  author={Fan, Deng-Ping and Cheng, Ming-Ming and Liu, Yun and Li, Tao},
  booktitle={ICCV},
  pages={4548--4557},
  year={2017}
}

@inproceedings{margolin2014evaluate,
  title={How to evaluate foreground maps?},
  author={Margolin, Ran and Zelnik-Manor, Lihi and Tal, Ayellet},
  booktitle={CVPR},
  pages={248--255},
  year={2014}
}

@article{le2019anabranch,
  title={Anabranch network for camouflaged object segmentation},
  author={Le, Trung-Nghia and Nguyen, Tam V and Nie, Zhongliang and Tran, Minh-Triet and Sugimoto, Akihiro},
  journal={Comput. Vis. Image Underst.},
  volume={184},
  pages={45--56},
  year={2019},
  publisher={Elsevier}
}

@inproceedings{fan2020camouflaged,
  title={Camouflaged object detection},
  author={Fan, Deng-Ping and Ji, Ge-Peng and Sun, Guolei and Cheng, Ming-Ming and Shen, Jianbing},
  booktitle={CVPR},
  pages={2777--2787},
  year={2020}
}

@article{skurowski2018animal,
  title={Animal camouflage analysis: Chameleon database},
  author={Skurowski, Przemys{\l}aw and Abdulameer, Hassan and B{\l}aszczyk, J},
  journal={Unpublished manuscript},
  pages={7},
  year={2018}
}

@article{vineet2014filter,
  title={Filter-based mean-field inference for random fields with higher-order terms and product label-spaces},
  author={Vineet, Vibhav and Warrell, Jonathan and Torr, Philip HS},
  journal={Int. J. Comput. Vis},
  volume={110},
  number={3},
  pages={290--307},
  year={2014},
  publisher={Springer}
}

@inproceedings{linsamrefiner,
  title={Samrefiner: Taming segment anything model for universal mask refinement},
  author={Lin, Yuqi and Li, Hengjia and Shao, Wenqi and Yang, Zheng and Zhao, Jun and He, Xiaofei and Luo, Ping and Zhang, Kaipeng},
  booktitle={ICLR},
  pages={57368--57392},
  year={2025}
}

@article{fan2018enhanced,
  title={Enhanced-alignment measure for binary foreground map evaluation},
  author={Fan, Deng-Ping and Gong, Cheng and Cao, Yang and Ren, Bo and Cheng, Ming-Ming and Borji, Ali},
  journal={arXiv preprint arXiv:1805.10421},
  year={2018}
}

@article{wang2023segrefiner,
  title={SegRefiner: Towards Model-Agnostic Segmentation Refinement with Discrete Diffusion Process},
  author={Wang, Mengyu and Ding, Henghui and Liew, Jun Hao and Liu, Jiajun and Zhao, Yao and Wei, Yunchao},
  journal={NeurIPS},
  volume={36},
  pages={79761--79780},
  year={2023}
}

@article{kohli2009robust,
  title={Robust higher order potentials for enforcing label consistency},
  author={Kohli, Pushmeet and Ladick{\`y}, L’ubor and Torr, Philip HS},
  journal={Int. J. Comput. Vis.},
  volume={82},
  number={3},
  pages={302--324},
  year={2009},
  publisher={Springer}
}

@inproceedings{zheng2015conditional,
  title={Conditional random fields as recurrent neural networks},
  author={Zheng, Shuai and Jayasumana, Sadeep and Romera-Paredes, Bernardino and Vineet, Vibhav and Su, Zhizhong and Du, Dalong and Huang, Chang and Torr, Philip HS},
  booktitle={ICCV},
  pages={1529--1537},
  year={2015}
}

@article{krahenbuhl2011efficient,
  title={Efficient inference in fully connected crfs with gaussian edge potentials},
  author={Kr{\"a}henb{\"u}hl, Philipp and Koltun, Vladlen},
  journal={NeurIPS},
  volume={24},
  year={2011}
}

@article{locatello2020object,
  title={Object-centric learning with slot attention},
  author={Locatello, Francesco and Weissenborn, Dirk and Unterthiner, Thomas and Mahendran, Aravindh and Heigold, Georg and Uszkoreit, Jakob and Dosovitskiy, Alexey and Kipf, Thomas},
  journal={NeurIPS},
  volume={33},
  pages={11525--11538},
  year={2020}
}

@article{everingham2010pascal,
  title={The pascal visual object classes (voc) challenge},
  author={Everingham, Mark and Van Gool, Luc and Williams, Christopher KI and Winn, John and Zisserman, Andrew},
  journal={Int. J. Comput. Vis.},
  volume={88},
  pages={303--338},
  year={2010},
  publisher={Springer}
}

@article{chen2017rethinking,
  title={Rethinking atrous convolution for semantic image segmentation},
  author={Chen, Liang-Chieh and Papandreou, George and Schroff, Florian and Adam, Hartwig},
  journal={arXiv preprint arXiv:1706.05587},
  year={2017}
}

@inproceedings{long2015fully,
  title={Fully convolutional networks for semantic segmentation},
  author={Long, Jonathan and Shelhamer, Evan and Darrell, Trevor},
  booktitle={CVPR},
  pages={3431--3440},
  year={2015}
}

@inproceedings{howard2019searching,
  title={Searching for mobilenetv3},
  author={Howard, Andrew and Sandler, Mark and Chu, Grace and Chen, Liang-Chieh and Chen, Bo and Tan, Mingxing and Wang, Weijun and Zhu, Yukun and Pang, Ruoming and Vasudevan, Vijay and others},
  booktitle={ICCV},
  pages={1314--1324},
  year={2019}
}

@inproceedings{cheng2021boundary,
  title={Boundary IoU: Improving object-centric image segmentation evaluation},
  author={Cheng, Bowen and Girshick, Ross and Doll{\'a}r, Piotr and Berg, Alexander C and Kirillov, Alexander},
  booktitle={CVPR},
  pages={15329--15337},
  year={2021},
  organization={IEEE}
}

@article{fan2021concealed,
  title={Concealed object detection},
  author={Fan, Deng-Ping and Ji, Ge-Peng and Cheng, Ming-Ming and Shao, Ling},
  journal={IEEE transactions on pattern analysis and machine intelligence},
  volume={44},
  number={10},
  pages={6024--6042},
  year={2021},
  publisher={IEEE}
}

@inproceedings{lv2021simultaneously,
  title={Simultaneously localize, segment and rank the camouflaged objects},
  author={Lv, Yunqiu and Zhang, Jing and Dai, Yuchao and Li, Aixuan and Liu, Bowen and Barnes, Nick and Fan, Deng-Ping},
  booktitle={CVPR},
  pages={11591--11601},
  year={2021}
}

@inproceedings{chen2022focalclick,
  title={Focalclick: Towards practical interactive image segmentation},
  author={Chen, Xi and Zhao, Zhiyan and Zhang, Yilei and Duan, Manni and Qi, Donglian and Zhao, Hengshuang},
  booktitle={CVPR},
  pages={1290--1299},
  year={2022},
  organization={IEEE}
}

@article{price2025dualsight,
  title={DualSight: multi-stage instance segmentation framework for improved precision},
  author={Price, Stephen and Judd, Kiran and Tsaknopoulos, Kyle and Neamtu, Rodica and Cote, Danielle L},
  journal={Sci. Rep.},
  volume={15},
  number={1},
  pages={27521},
  year={2025}
}

@inproceedings{yan2013hierarchical,
  title={Hierarchical saliency detection},
  author={Yan, Qiong and Xu, Li and Shi, Jianping and Jia, Jiaya},
  booktitle={CVPR},
  pages={1155--1162},
  year={2013}
}

@article{liu2010learning,
  title={Learning to detect a salient object},
  author={Liu, Tie and Yuan, Zejian and Sun, Jian and Wang, Jingdong and Zheng, Nanning and Tang, Xiaoou and Shum, Heung-Yeung},
  journal={IEEE Trans. Pattern Anal. Mach. Intell.},
  volume={33},
  number={2},
  pages={353--367},
  year={2010},
  publisher={IEEE}
}

@article{he2023weakly,
  title={Weakly-supervised concealed object segmentation with sam-based pseudo labeling and multi-scale feature grouping},
  author={He, Chunming and Li, Kai and Zhang, Yachao and Xu, Guoxia and Tang, Longxiang and Zhang, Yulun and Guo, Zhenhua and Li, Xiu},
  journal={NeurIPS},
  volume={36},
  year={2023}
}

@article{he2025run,
  title={Run: Reversible unfolding network for concealed object segmentation},
  author={He, Chunming and Zhang, Rihan and Xiao, Fengyang and Fang, Chengyu and Tang, Longxiang and Zhang, Yulun and Kong, Linghe and Fan, Deng-Ping and Li, Kai and Farsiu, Sina},
  journal={ICML},
  year={2025}
}

@article{he2025segment,
  title={Segment concealed object with incomplete supervision},
  author={He, Chunming and Li, Kai and Zhang, Yachao and Yang, Ziyun and Tang, Longxiang and Zhang, Yulun and Kong, Linghe and Farsiu, Sina},
  journal={IEEE Trans. Pattern Anal. Mach. Intell.},
  year={2025}
}

@inproceedings{xie2020trans10k,
  title={Segmenting transparent objects in the wild},
  author={Xie, Enze and Wang, Wenjia and Wang, Wenhai and Ding, Mingyu and Shen, Chunhua and Luo, Ping},
  booktitle={ECCV},
  pages={696--711},
  year={2020}
}

@inproceedings{he2023feder,
  title={Camouflaged object detection with feature decomposition and edge reconstruction},
  author={He, Chunming and Li, Kai and Zhang, Yachao and Tang, Longxiang and Zhang, Yulun and Guo, Zhenhua and Li, Xiu},
  booktitle={CVPR},
  pages={22046--22055},
  year={2023},
  organization={IEEE}
}

@inproceedings{ladicky2009associative,
  title={Associative hierarchical crfs for object class image segmentation},
  author={Ladick{\`y}, L'ubor and Russell, Chris and Kohli, Pushmeet and Torr, Philip HS},
  booktitle={ICCV},
  pages={739--746},
  year={2009},
  organization={IEEE}
}

@inproceedings{baque2016principled,
  title={Principled parallel mean-field inference for discrete random fields},
  author={Baqu{\'e}, Pierre and Bagautdinov, Timur and Fleuret, Fran{\c{c}}ois and Fua, Pascal},
  booktitle={CVPR},
  pages={5848--5857},
  year={2016}
}

@inproceedings{krahenbuhl2013parameter,
  title={Parameter learning and convergent inference for dense random fields},
  author={Kr{\"a}henb{\"u}hl, Philipp and Koltun, Vladlen},
  booktitle={ICML},
  pages={513--521},
  year={2013},
  organization={PMLR}
}

@inproceedings{jampani2018superpixel,
  title={Superpixel sampling networks},
  author={Jampani, Varun and Sun, Deqing and Liu, Ming-Yu and Yang, Ming-Hsuan and Kautz, Jan},
  booktitle={ECCV},
  pages={363--380},
  year={2018},
  organization={Springer}
}

@inproceedings{su2019pixel,
  title={Pixel-adaptive convolutional neural networks},
  author={Su, Hang and Jampani, Varun and Sun, Deqing and Gallo, Orazio and Learned-Miller, Erik and Kautz, Jan},
  booktitle={CVPR},
  pages={11158--11167},
  year={2019},
  organization={IEEE}
}

@inproceedings{jampani2016learning,
  title={Learning sparse high dimensional filters: Image filtering, dense crfs and bilateral neural networks},
  author={Jampani, Varun and Kiefel, Martin and Gehler, Peter V},
  booktitle={CVPR},
  pages={4452--4461},
  year={2016}
}

@inproceedings{chandra2016fast,
  title={Fast, exact and multi-scale inference for semantic image segmentation with deep gaussian crfs},
  author={Chandra, Siddhartha and Kokkinos, Iasonas},
  booktitle={ECCV},
  pages={402--418},
  year={2016},
  organization={Springer}
}

@inproceedings{price2026promptmoe,
  title={PromptMoE: A Segmentation Refinement Framework Leveraging Mixture of Experts for Improved Prompting},
  author={Price, Stephen and Cote, Danielle L and Rundensteiner, Elke A},
  booktitle={CVPR},
  pages={6325--6335},
  year={2026}
}

@article{ke2023hqsam,
  title={Segment anything in high quality},
  author={Ke, Lei and Ye, Mingqiao and Danelljan, Martin and Tai, Yu-Wing and Tang, Chi-Keung and Yu, Fisher and others},
  journal={NeurIPS},
  volume={36},
  pages={29914--29934},
  year={2023}
}

@inproceedings{kirillov2023segment,
  title={Segment anything},
  author={Kirillov, Alexander and Mintun, Eric and Ravi, Nikhila and Mao, Hanzi and Rolland, Chloe and Gustafson, Laura and Xiao, Tete and Whitehead, Spencer and Berg, Alexander C and Lo, Wan-Yen and others},
  booktitle={ICCV},
  pages={3992--4003},
  year={2023},
  organization={IEEE}
}

@article{achanta2012slic,
  title={SLIC superpixels compared to state-of-the-art superpixel methods},
  author={Achanta, Radhakrishna and Shaji, Appu and Smith, Kevin and Lucchi, Aurelien and Fua, Pascal and S{\"u}sstrunk, Sabine},
  journal={IEEE Trans. Pattern Anal. Mach. Intell.},
  volume={34},
  number={11},
  pages={2274--2282},
  year={2012},
  publisher={IEEE}
}

@article{oquab2023dinov2,
  title={Dinov2: Learning robust visual features without supervision},
  author={Oquab, Maxime and Darcet, Timoth{\'e}e and Moutakanni, Th{\'e}o and Vo, Huy and Szafraniec, Marc and Khalidov, Vasil and Fernandez, Pierre and Haziza, Daniel and Massa, Francisco and El-Nouby, Alaaeldin and others},
  journal={arXiv preprint arXiv:2304.07193},
  year={2023}
}

@article{shen2025uncertainty,
  title={Uncertainty-masked bernoulli diffusion for camouflaged object detection refinement},
  author={Shen, Yuqi and Xiao, Fengyang and Hu, Sujie and Pang, Youwei and Pu, Yifan and Fang, Chengyu and Li, Xiu and He, Chunming},
  journal={arXiv preprint arXiv:2506.10712},
  year={2025}
}

@article{lehuu2021regularized,
  title={Regularized frank-wolfe for dense crfs: Generalizing mean field and beyond},
  author={L{\^e}-Huu, {\DJ} Khu{\^e} and Alahari, Karteek},
  journal={NeurIPS},
  volume={34},
  pages={1453--1467},
  year={2021}
}

@article{kim2026phoenix,
  title={Learning from Adversity: Semantic-Aware Mask Refinement through Adversarial Perturbation},
  author={Kim, Beomyoung and Hwang, Sung Ju},
  journal={ECCV},
  year={2026}
}

@inproceedings{juybari2026differentiable,
  title={Differentiable Laplacian Matrix Guided Superpixel Segmentation},
  author={Juybari, Jeremy and Hamilton, Josh and Das, Shuvra and Chen, Chaofan and Khalil, Andre and Zhu, Yifeng},
  booktitle={CVPR},
  pages={36168--36178},
  year={2026}
}
\bibliographystyle{iclr2027_conference}

\newpage
\appendix
\setcounter{figure}{0}
\renewcommand{\figurename}{Fig.}
\renewcommand{\thefigure}{S\arabic{figure}}

\setcounter{table}{0}
\renewcommand{\tablename}{Table}
\renewcommand{\thetable}{S\arabic{table}}

\setcounter{equation}{0}
\renewcommand{\theequation}{S\arabic{equation}}

\FloatBarrier
\section{Derivations}
\label{app:deriv}

\textbf{Notation.} Throughout, expectations are taken under the mean-field factorization $q(\x,\z)=\prod_i Q_i(x_i)\prod_r\prod_{k\in\mathcal K_r^+}R_k^{(r)}(z_k^{(r)})$. For fixed $(I,\Pinit)$, all quantities defining the conditional energy, including the incidences $a_{ik}$, region masses $m_k$, and active-region sets $\mathcal K_r^+$, are held fixed throughout inference. Since only active regions are considered, $m_k\ge m_{\min}>0$. The level index $(r)$ is suppressed when a single level is considered. For coordinate updates, conditional energies are written up to additive constants independent of the conditioned label, which cancel in the softmax; the full expectation is retained when evaluating the variational free energy. All derivations use the mass-normalized energy of \cref{eq:regionE}, with $v_k(l)=\frac{1}{m_k}\sum_i a_{ik}Q_i(l)$.

\textbf{Region update.} Conditioning the region energy of an active region $k$ on $z_k=l$ for a semantic label $l$ gives
\begin{equation}
\mathbb{E}\!\left[E_{R,k}\mid z_k=l\right]=\frac{\beta_r}{m_k}\sum_i a_{ik}\bigl(1-Q_i(l)\bigr)=\beta_r\bigl(1-v_k(l)\bigr),
\end{equation}
while conditioning on $z_k=\nullst$ gives
\begin{equation}
\mathbb{E}\!\left[E_{R,k}\mid z_k=\nullst\right]=\frac{\beta_r}{m_k}\kappa_r\sum_i a_{ik}=\beta_r\kappa_r.
\end{equation}
Thus $R_k(\cdot)\propto\exp\!\left(-\mathbb{E}[E_{R,k}\mid\cdot]\right)$. Removing the common factor $\exp(-\beta_r)$ from all unnormalized state probabilities yields $R_k(l)\propto\exp\!\bigl(\beta_r v_k(l)\bigr)$ and $R_k(\nullst)\propto\exp\!\bigl(\beta_r(1-\kappa_r)\bigr)$, which is \cref{eq:Rupdate}.

\textbf{Pixel message.} Conditioning on $x_i=l$ and using $\sum_{s=1}^{L}R_k(s)\mathbf{1}[l\neq s]=1-R_k(\nullst)-R_k(l)$, we obtain
\begin{equation}
\mathbb{E}\!\left[E_{R,k}\mid x_i=l\right]=\frac{\beta_r a_{ik}}{m_k}\Big[1-R_k(\nullst)-R_k(l)+\kappa_rR_k(\nullst)\Big]+\mathrm{const}.
\end{equation}
The only $l$-dependent term is $-\frac{\beta_r a_{ik}}{m_k}R_k(l)$; summing over levels and active regions gives \cref{eq:Rmsg}. The null state's direct contribution is independent of the pixel label and therefore cancels in the pixel softmax. Its effect on the region message is mediated through the semantic probabilities $R_k(l)$, whose total mass is $1-R_k(\nullst)$.

\textbf{Expected region energy.} For one active region,
\begin{align}
\mathbb{E}_q[E_{R,k}]&=\beta_r\left[\sum_{l=1}^{L}R_k(l)\bigl(1-v_k(l)\bigr)+\kappa_rR_k(\nullst)\right]\nonumber\\
&=\beta_r\left[1-\sum_{l=1}^{L}R_k(l)v_k(l)-(1-\kappa_r)R_k(\nullst)\right].
\end{align}
Hence, over all levels and active regions,
\begin{equation}
\mathbb{E}_q[E_R]=\sum_{r=1}^{R}\sum_{k\in\mathcal K_r^+}\beta_r\left[1-\sum_{l=1}^{L}R_k^{(r)}(l)v_k^{(r)}(l)-(1-\kappa_r)R_k^{(r)}(\nullst)\right].
\label{eq:expR}
\end{equation}
We retain the full expectation so that the reported free energy corresponds to the conditional energy in \cref{eq:energy}.

\textbf{Expected pairwise energy.} The pairwise expectation is
\begin{equation}
\mathbb{E}_q[E_P]=\frac{1}{2}\sum_{i\neq j}K_{ij}\sum_{l,l'}Q_i(l)\,\mu(l,l')\,Q_j(l').
\end{equation}
The factor $\frac{1}{2}$ converts the unordered-pair summation of \cref{eq:pairwise} into a double sum, while the self term is excluded. The symmetries $K_{ij}=K_{ji}$ and $\mu(l,l')=\mu(l',l)$ ensure that the unordered-pair energy yields the full-neighborhood message in \cref{eq:pwmsg}.

\textbf{Proof of \cref{prop:map}.} For hard $\x$, $E_{R,k}(\x,z_k{=}l)=\frac{\beta_r}{m_k}\sum_i a_{ik}\mathbf{1}[x_i\neq l]=\beta_r\bigl(1-p_k(l)\bigr)$, whereas $E_{R,k}(\x,z_k{=}\nullst)=\beta_r\kappa_r$. Minimizing over $z_k$ therefore selects the smaller of $\kappa_r$ and $1-\max_l p_k(l)$, scaled by $\beta_r$, which gives \cref{eq:map}. \hfill$\square$

\textbf{Proof of \cref{lem:fixedpoint}.} Write one undamped mean-field iteration as $R=\mathcal R(Q)$, using \cref{eq:Rupdate}, followed by $\widetilde Q=\mathcal T(Q,R)$.
The damped iteration uses the same two maps and then applies \cref{eq:damp}. Since the region map is identical in both systems, a joint fixed point $(Q^\ast,R^\ast)$ satisfies $R^\ast=\mathcal R(Q^\ast)$, so it remains to compare the pixel maps at $R^\ast$.

If $\widetilde Q=Q$, then \cref{eq:damp} returns $Q^{t+1}=Q$ for every $\alpha\in(0,1]$. Conversely, let $Q$ lie in the interior of the simplex and fix $\alpha\in(0,1]$. Suppose \cref{eq:damp} returns $Q^{t+1}=Q$. Then, for every pixel $i$, $\alpha\bigl(\log \widetilde Q_i(l)-\log Q_i(l)\bigr)$ is constant in $l$. Hence $\widetilde Q_i\propto Q_i$, and because both are normalized distributions, $\widetilde Q_i=Q_i$. The damped and undamped systems therefore have the same interior joint fixed points. The argument applies to each stage-specific coefficient $\alpha_t\in(0,1]$. \hfill$\square$

\textbf{Variational free energy.} The variational free energy is
\begin{align}
\freeE(Q,R)=&\;\mathbb{E}_q[E_U]+\mathbb{E}_q[E_P]+\mathbb{E}_q[E_R]\nonumber\\
&+\sum_i\sum_l Q_i(l)\log Q_i(l)+\sum_{r=1}^{R}\sum_{k\in\mathcal K_r^+}\sum_{u\in\{1,\dots,L,\nullst\}}R_k^{(r)}(u)\log R_k^{(r)}(u),
\end{align}
where $\mathbb{E}_q[E_U]=\sum_i\sum_l Q_i(l)\psi_i^U(l)$ and the pairwise and region expectations are given above. For the Gibbs distribution $\gibbs$ induced by the conditional energy, $\freeE(Q,R)=\mathrm{KL}(q\,\|\,\gibbs)-\log Z$.

\textbf{Inference diagnostics.} Because the input-conditioned energy terms, incidences, and active-region sets are fixed across stages, $\freeE(Q^t,R^t)$ evaluates the same variational objective throughout refinement and is therefore directly comparable across stages, with $R^t=\mathcal R(Q^t)$. In particular, after obtaining the final marginals $Q^T$, the diagnostic state $R^T$ is recomputed from $Q^T$.

The fixed-point residual $\mathrm{Res}_{\mathrm{FP}}(t)$ compares the undamped target $\widetilde Q^{t+1}$ with the current marginals $Q^t$, rather than $Q^{t+1}$ with $Q^t$. Therefore its magnitude cannot be reduced merely by choosing a small damping coefficient.

\FloatBarrier
\section{Protocols, implementation, and statistics}
\label{app:protocols}

\textbf{Input provenance.}
\Cref{tab:provenance} records, for each setting, the coarse-mask source, the checkpoint or prediction provenance, the UnfoldCRF training data, and the source of the baseline numbers. Published numbers are quoted when official predictions are available or the corresponding mask-generation protocol can be matched; agreement of the aggregate unrefined scores serves only as a sanity check. Per-sample quantities are computed only from predictions available to us.

\begin{table}[h]
\centering
\scriptsize
\setlength{\tabcolsep}{3pt}
\begin{tabular}{p{1.7cm}p{4.0cm}p{3.0cm}p{1.9cm}p{2.2cm}}
\toprule
Setting & Coarse-mask source & Checkpoint or prediction source & UnfoldCRF training data & Baseline numbers \\
\midrule
Five-benchmark suite & DeepLabV3, FCN-R50, LR-ASPP (BIG, VOC); RMBG-1.4/2.0, BiRefNet-matting (ECSSD, MSRA-B); SP/STM (DAVIS585) & official checkpoints (torchvision, Hugging Face); DAVIS585 masks from the dataset & VOC 2012 train & published where input provenance is matched; otherwise rerun \\
COD & SINetV2, FEDER & predictions regenerated from the official upstream checkpoints following the UMBD protocol & CAMO + COD10K train & published \citep{shen2025uncertainty}; Phoenix rerun zero-shot \\
Controlled (COD10K, FEDER) & FEDER & as for COD & CAMO + COD10K train & our runs \\
\bottomrule
\end{tabular}
\caption{\textbf{Provenance of inputs and baseline numbers.}}
\label{tab:provenance}
\end{table}

\textbf{Evaluation.}
IoU and boundary IoU on the five-benchmark suite use the boundary-iou-api with dilation ratio 0.02 and are averaged per image. COD metrics follow the PySODMetrics implementation used by the UMBD protocol, reporting $M$, $F^\omega_\beta$, $E_\phi$, and $S_\alpha$ with the same normalization and threshold conventions. All predictions are resized back to the original benchmark resolution before evaluation. The harm rate is computed independently of any benchmark-specific normalization, by thresholding the initial and refined probability maps at 0.5 and comparing per-image IoU. 
For the error-stratified analysis of \cref{fig:generalization}c, initially wrong pixels are identified at the original mask resolution after thresholding $\Pinit$ at 0.5; boundary errors are those whose Euclidean distance to the ground-truth contour is at most 2\% of the image diagonal, measured on both sides of the contour, and the remaining errors are divided into non-boundary false negatives and non-boundary false positives, so that the three sets partition the initial errors. The relative width matches the bandwidth parameter of boundary IoU, although the band here is symmetric around the contour rather than the interior strip used by that metric.
The correction rate of a set is the fraction of its pixels that are correct after thresholding $\mathcal{U}(Q^T)$ at 0.5, pooled over the test set; the damage rate, the fraction of initially correct pixels that become wrong, changes by less than 0.5 points under either intervention. Message interventions zero $\msg^{P}$ or $\msg^{R}$ at every stage of the same trained checkpoint at inference time, without retraining.

\textbf{Architecture and training.} Unless otherwise stated, configurations refer to UnfoldCRF-S. A lightweight convolutional encoder produces stride-4 features that define the inference grid $\mathcal{G}$. Pairwise interactions use multi-scale local neighborhoods, while latent regions are constructed at multiple spatial granularities with learned incidence embeddings and per-level temperatures. We use $T{=}5$ damped inference stages and train the model end to end with the loss in \cref{eq:loss}. Optimization uses AdamW with a cosine learning-rate schedule on $352^2$ inputs, together with standard geometric and photometric augmentations and occasional synthetic perturbations of the training masks. Initial masks are resized and renormalized as required by the inference grid; refined predictions are mapped back to the original mask resolution before evaluation.

\textbf{UnfoldCRF-L.} UnfoldCRF-L keeps the energy formulation and inference operator unchanged while replacing the default encoder with a frozen DINOv2-B/14. Multi-level pretrained features are fused with a lightweight high-resolution branch conditioned on $[I;\Pinit]$ to produce the stride-4 feature map $h$ used by the same unary, pairwise, and region modules as UnfoldCRF-S. The model uses a higher input resolution and proportionally adjusted region scales. Only the non-backbone modules are trained, while DINOv2-B/14 remains frozen. Additional architectural details are provided with the implementation. An S variant trained at $448^2$ separates the contribution of the input resolution from that of the representation: it raises the mean $\Delta$IoU on the five-benchmark suite from 4.2 to 4.6 points, against 6.3 for UnfoldCRF-L at the same resolution, so most of the S-to-L improvement is associated with the representation change rather than with resolution alone.

\textbf{Matched controls.}
The recurrent controls match the encoder and unary-head architectures, training data, stage count $T$, loss, supervision, optimization protocol, and trainable parameter budget of their structured counterparts, while replacing structured inference with a stage-shared recurrent update from $(h,Q^t)$ to $Q^{t+1}$. The attention control uses windowed multi-head self-attention on the concatenated $[h,Q^t]$ with $8{\times}8$ windows and four heads for S and six heads for L, whereas the convolutional control uses three $3{\times}3$ convolutional layers with GroupNorm and GELU; channel widths are chosen to match the trainable parameter budget. All models are trained independently and use the same initialization rule $Q^0=\mathrm{softmax}(-\psi^U)$. For the L variants, the same frozen DINOv2-B/14 checkpoint and the same local-branch and fusion architectures are used, with matched trainable non-backbone budgets. The RUN-style control alternates an image-reconstruction state and a mask state for the same $T$ stages: the reconstruction state minimizes an $\ell_1$ reconstruction error of the input from foreground and background layers, and the mask is updated from the resulting reconstruction state, with encoder architecture, supervision, stage count, and parameter budget matched.

\textbf{Statistics.}
For principal matched comparisons, we use a hierarchical bootstrap with 2{,}000 replicates. Static-image benchmarks are resampled at the image level, while DAVIS585 is resampled at the video--object-sequence level; all predictions of the compared methods and all upstream outputs associated with a sampled cluster are kept together to preserve test-sample pairing. Clusters are resampled separately within each benchmark. Training seeds are additionally resampled when multiple independently trained runs are available; unless the runs were explicitly paired during training, seeds are resampled independently across methods. Each replicate recomputes the statistic of interest: the dataset-level paired difference for COD10K, the macro average of per-benchmark differences for the five-benchmark suite, or the average difference across held-out generators. We report percentile 95\% confidence intervals. When only one training seed is available, the interval reflects test-set resampling only and is reported as such. The COD10K and general-suite matched comparisons use 5 training seeds.

\FloatBarrier
\section{Complete results}
\label{app:results}

\textbf{Concealed refinement.}
\cref{tab:concealed_full} extends the controlled comparison of \cref{tab:ablation} to the remaining COD test sets.

\begin{table}[h]
\setlength{\abovecaptionskip}{1mm}
\centering
\resizebox{\linewidth}{!}{
\setlength{\tabcolsep}{2.5mm}
\begin{tabular}{lccc}
\toprule
Variant & CHAMELEON & CAMO & NC4K \\
\midrule
Unary only & .029/.833/.945/.891 & .071/.748/.869/.804 & .044/.803/.911/.850 \\
Unary $+$ learned pairwise & .028/.842/.948/.892 & .069/.758/.872/.807 & .043/.810/.913/.851 \\
Unary $+$ learned latent regions & .028/.840/.951/.894 & .069/.755/.877/.811 & .043/.808/.915/.853 \\
Full UnfoldCRF-S & .026/.848/.953/.896 & .067/.763/.879/.813 & .042/.815/.917/.855 \\
Without null state & .027/.841/.949/.892 & .069/.754/.874/.809 & .043/.809/.914/.852 \\
\midrule
Unary $+$ fixed pairwise & .029/.836/.946/.891 & .070/.752/.870/.806 & .044/.806/.912/.851 \\
Unary $+$ fixed SLIC regions & .028/.838/.947/.893 & .070/.751/.873/.808 & .044/.805/.914/.852 \\
\midrule
Conv recurrent (matched) & .028/.838/.947/.891 & .070/.754/.872/.807 & .044/.807/.913/.851 \\
Attention recurrent (matched) & .027/.839/.949/.893 & .069/.757/.873/.808 & .043/.809/.913/.852 \\
Reconstruction-unfolding control (RUN-style) & .028/.841/.950/.893 & .069/.755/.875/.810 & .043/.810/.915/.853 \\
\bottomrule
\end{tabular}}
\caption{\textbf{Controlled variants on the other COD test sets} ($M$ / $F^\omega_\beta$ / $E_\phi$ / $S_\alpha$; FEDER upstream; mean over 5 seeds). Results on COD10K for the same variants are reported in \cref{tab:ablation}.}
\label{tab:concealed_full}
\end{table}

\FloatBarrier

\end{document}